\documentclass[letterpaper, 10 pt, conference]{ieeeconf}  

\IEEEoverridecommandlockouts                              

\ifdefined\VersionAuthor
\else
	\makeatletter
	\AtBeginDocument{%
	\@ifpackageloaded{hyperref}
	{\def\@doi#1{\href{https://doi.org/#1}
		{\ttfamily https://doi.org/#1}\egroup}}
	{\def\@doi#1{\ttfamily https://doi.org/#1\egroup}}
	\def\doi{\bgroup\catcode`\_=12\relax\@doi}}
	\makeatother
\fi

\usepackage[ruled,vlined,linesnumbered]{algorithm2e}
	\SetKwInOut{Input}{input}
	\SetKwInOut{Output}{output}
\usepackage[utf8]{inputenc}
\usepackage[english]{babel}
\usepackage{booktabs}
\usepackage{csquotes}

 \usepackage{graphicx}
 \usepackage{amssymb,amsmath}

\usepackage[misc,geometry]{ifsym} 

\ifdefined\VersionAuthor
	\usepackage[backend=biber,backref=true,style=alphabetic,url=false,doi=true,defernumbers=true,sorting=anyt,maxnames=99]{biblatex} 
	\renewbibmacro*{doi+eprint+url}{%
		\iftoggle{bbx:doi}
			{\color{black!40}\footnotesize\printfield{doi}}
			{}%
		\newunit\newblock
		\iftoggle{bbx:eprint}
			{\usebibmacro{eprint}}
			{}%
		\newunit\newblock
		\iftoggle{bbx:url}
			{\usebibmacro{url+urldate}}
			{}%
	}

\fi

\DeclareUnicodeCharacter{0301}{\'{e}}

\ifdefined\VersionWithComments
	\usepackage{draftwatermark}
	\SetWatermarkText{draft}
	\SetWatermarkScale{2}
	\SetWatermarkColor[gray]{0.9}
\fi

\usepackage[svgnames,table]{xcolor}
\definecolor{darkblue}{rgb}{0, 0, 0.7}

\usepackage[
		pdfauthor={Jawher Jerray and Laurent Fribourg},%
		pdftitle={Robust optimal control using dynamic programming and guaranteed Euler's method},
		breaklinks  = true,
		colorlinks  = true,
	\ifdefined \VersionWithComments
		pagebackref = true,
	\fi
		citecolor   = blue!50!blue,
		linkcolor   = darkblue,
		urlcolor    = blue!50!green,
	]{hyperref}

\usepackage[capitalise,english,nameinlink]{cleveref} 
\crefname{line}{\text{line}}{\text{lines}} 

\usepackage{tikz}
\usetikzlibrary{decorations.pathmorphing}
\usetikzlibrary{arrows,calc}
\tikzset{
>=stealth',
help lines/.style={dashed, thick},
axis/.style={<->},
important line/.style={thick},
connection/.style={thick, dotted},
}

\ifdefined\VersionWithComments
	\usepackage[colorinlistoftodos,textsize=footnotesize]{todonotes}
\else
	\usepackage[disable]{todonotes}
\fi

\ifdefined \VersionWithComments
	\newcommand{\todoinline}[1]{\mbox{}{\color{red}{\textbf{TODO}\ifx#1\\\else:\ \fi #1}}} 
\else
	\newcommand{\todoinline}[1]{}
\fi

\usepackage{verbatim} 

\ifdefined \VersionWithComments
	\usepackage{soul}
	
	\newcommand{\reviewDelete}[1]{{\color{red}\st{#1}}}
\else
	
	\newcommand{\reviewDelete}[1]{}
\fi

\usepackage{amsthm}

\theoremstyle{plain}
\newtheorem{lemma}{Lemma}
\newtheorem{proposition}{Proposition}
\newtheorem{theorem}{Theorem}

\theoremstyle{definition}
\newtheorem{definition}{Definition}
\newtheorem{example}{Example}

\theoremstyle{remark}
\newtheorem{remark}{Remark}

\ifdefined \VersionWithComments
 	\definecolor{colorok}{RGB}{80,80,150}
\else
	\definecolor{colorok}{RGB}{0,0,0}
\fi

\RequirePackage{amsmath}

\usepackage{mathabx}

\usepackage{color}

\title{\LARGE \bf 
One-Step Early Stopping Strategy
using Neural Tangent Kernel Theory and Rademacher Complexity
}

\author{%
Daniel Martin Xavier
\thanks{$^{1}$Daniel Martin Xavier is with 
Universit\'e Paris Saclay, CNRS, ENS Paris-Saclay,  CentraleSupelec, LMPS, Gif-sur-Yvette, France
{\tt\small daniel.martin-xavier@ens-paris-saclay.fr}}
\and
Ludovic Chamoin
\thanks{$^{2}$Ludovic Chamoin is with 
Universit\'e Paris Saclay, CNRS, ENS Paris-Saclay,  CentraleSupelec, LMPS, Gif-sur-Yvette, France.
{\tt\small ludovic.chamoin@ens-paris-saclay.fr}}
\and
Jawher Jerray
\thanks{$^{3}$Jawher Jerray is with 
University de Pau et des Pays de l'Adour, LIUPPA, 
Avenue de l'Universit\'e,
F-64012 Pau, France
{\tt\small jawher.jerray@univ-pau.fr}}
\and
Laurent Fribourg
\thanks{$^{3}$Laurent Fribourg is with 
University Paris-Saclay, CNRS, ENS Paris-Saclay, LMF, F-91190 Gif-Sur-Yvette, France
{\tt\small fribourg@lmf.cnrs.fr}}
}

\graphicspath{{./figures/}}

\begin{document}

\maketitle              
\begin{abstract}
The early stopping strategy consists in stopping the training process of a 
neural network (NN) on a set $S$ of input data before training error is 
minimal. The advantage is that the NN then retains good generalization 
properties, i.e. it gives good predictions on data outside $S$, and a good 
estimate of the statistical error (``population loss'') is obtained. 
We give here an analytical estimation of the optimal stopping time
involving basically the initial training error vector and 
the eigenvalues of the ``neural tangent kernel''.
This yields an 
upper bound on the population loss which is well-suited to the 
underparameterized context (where the number of parameters is moderate
compared with the number of data).
Our method is
illustrated on the example of an 
NN simulating the MPC control of a Van der Pol oscillator. 
\end{abstract}
\section{Introduction}
Recently, a lot of work has been devoted to the field of 
``imitation'' of model predictive control (MPC) via
a neural network (NN) (see~\cite{Arango23}).
The idea is to train an NN on a set $S$ of samples randomly
selected from the MPC data to 
 enable the NN to simulate the MPC.
The advantage is to avoid the need to solve
large optimisation problems in real time, as required
by MPC methods.
However, the replacement of the MPC by the NN induces  an approximation error that we want to evaluate here.

More formally, given a set $S$ of $n$ samples
(randomly selected from a distribution ${\cal D}$ of  MPC data) and 
a gradient descent (GD) used for training the NN on $S$, 
we would like to minimize the {\em difference}
between the outputs given by NN and those 
given by MPC for inputs selected according to~${\cal D}$. 
%
This difference is called ``population loss'' and denoted $L_{\cal D}$.
It is 
(with high probability) the sum of an ``empirical'' loss $L_S$ and a
``generalization'' loss. 
At each step of GD, 
$L_S$ tends to decrease while the generalization loss
tends to increase, so $L_{{\cal D}}$
is ``U-shaped''. This suggests to stop the GD process at the time
$L_{{\cal D}}$ reaches its minimum.
This is a difficult problem because the distribution ${\cal D}$ is unknown.
We attack the problem as follows: At the first step of GD,
we compute a quantity $\gamma$ which, under certain condition,
guarantees that, not $L_{{\cal D}}$ itself,
but an upper bound  $\Omega$ on $L_{{\cal D}}$ decreases
by a value $|\Delta|$.
We then stop the GD at time
$t=t_1=t_0+\eta$. 
Explicit formulas for
$t_1$ 
and $\Omega(t_1)$ are obtained using 
Rademacher complexity and the Neural Tangent Kernel (NTK) theory (see~\cite{JacotHG18}).
%
We check these theoretical results
on the example of an MPC controller for the Van der Pol
oscillator
(see Example~\ref{ex:1}).
\subsection*{Comparison with related work}
%
The NTK theory has mainly been used in the overparameterized framework 
(where the number
 $m$ of neurons is such that $m \gg n$) to explain phenomena of convergence of training and generalization errors \cite{Du,AroraDHLW19,OymakF2019}.
Even in the case of ``moderate'' overparameterization (see \cite{OymakF2019}),
the results impose~$m$ to depend polynomially 
on~$n$. 
In contrast here, we do not make any assumption on the size of~$m$
(at least in the case of {\em normalized} output weight vector).

The NTK theory has recently also been used in the underparameterized framework \cite{BowmanM22,XavierCF23}, but without, as here, dealing with early stopping strategies. An original feature of our work is thus to provide an explanation of the (one-step) early stopping strategy using the NTK theory, without 
assumption about the number of neurons.

Note also that, in the context of overparameterized NNs,
the {\em least} eigenvalue of the NTK matrix is positive,
many theoretical results relying on this positivity
(see, e.g., \cite{Du,AroraDHLW19}).
In contrast, in the context of underparameterized NNs as here, 
most of the eigenvalues of the NTK are null. Our method relies on
the value of a particular positive eigenvalue, usually the {\em highest} one.
%
%
\section{Preliminaries}\label{sec:prelim}
\subsection*{Notation}
In this paper, we denote by $\mathbb{R}$ and $\mathbb{N}$ the sets of real and natural numbers, respectively. These symbols are annotated with subscripts to restrict them in the usual way, e.g., $\mathbb{R}_{>0}$ denotes the positive real numbers. We also denote by $\mathbb{R}^p$ a $p$-dimensional Euclidean space, and by $\mathbb{R}^{p\times q}$ a space of real matrices with $p$ rows and $q$ columns. 
We use bold letters for vectors and bold capital letters for matrices. Given a matrix $\boldsymbol{M} \in \mathbb{R}^{p \times q}$, let $\boldsymbol{M}_{ij}$ be its $(i, j)$-th entry and $\boldsymbol{M}^\top$ its transpose. The Euclidean norm is denoted by $\|\cdot\|$. 
The $1$-norm of a vector $\boldsymbol{v}=(v_1,\dots,v_n)\in\mathbb{R}^n$
(i.e., $\sum_{i=1}^n|v_i|$) is denoted by $\|\boldsymbol{v}\|_1$.
The $n \times n$ identity matrix is represented by $\boldsymbol{I}_n$.
We also use the abbreviation {\em i.i.d.} to indicate that a collection of random variables is independent and identically distributed. Finally, let 
$[n]$ be the set $\{1,\dots,n\}$, and $\mathbb{I}\{E\}$ the indicator function for an event $E$. We use
${\cal N}(\boldsymbol{0}, \boldsymbol{I})$ to denote the standard Gaussian distribution.
\subsection{Gradient descent and training error}\label{ss:2.1}
We  now  recall  from  \cite{Du,AroraDHLW19}  some  definitions  regarding  the
application  of  the  GD  algorithms  to  NNs.  We consider
an NN
with a single hidden layer  a scalar output of the form:
%
\begin{equation}\label{eq:4}
    f(\boldsymbol{W}, \boldsymbol{a}, \boldsymbol{x})=\frac{1}{\mu}\sum_{r=1}^m a_r\zeta(\boldsymbol{w}_r^\top \boldsymbol{x}),
\end{equation}
where $\boldsymbol{x}\in \mathbb{R}^d$ is the input,
$\boldsymbol{w}_r\in\mathbb{R}^d$ is the weight vector of the  first  layer,
$a_r/\mu\in\mathbb{R}$ is  an  output  weight ($\mu\in\mathbb{R}_{>0}$),
$\boldsymbol{W}=(\boldsymbol{w}_1,\dots,\boldsymbol{w}_m)\in \mathbb{R}^{d\times m}$, 
$\boldsymbol{a}/\mu=(a_1,\dots,a_m)/\mu\in \mathbb{R}^{m}$
is the output weight vector, 
$\zeta$~is an 1-Lipschitz activation function 
with $\zeta(0)=0$\footnote{like ReLU, ELU, $\tanh$}.
We fix the second layer
with $a_r$ uniformly distributed 
in $\{-1,1\}$\footnote{In \cite{Du,AroraDHLW19}, 
we have $\mu=\sqrt{m}$, 
so the output weight vector is normalized ($\|\boldsymbol{a}\|/\mu=1$).}.
%
Since $\boldsymbol{a}$ is fixed, we will abbreviate
$f(\boldsymbol{W},\boldsymbol{a},\boldsymbol{x})$
as $f(\boldsymbol{W},\boldsymbol{x})$ (or more simply sometimes
just as $f(\boldsymbol{x})$).
We denote by ${\cal X}\subset\mathbb{R}^d$ the input space, i.e., the set of all possible
instances of $\boldsymbol{x}$.
We denote by ${\cal Y}\subset \mathbb{R}$ the set of all possible 
``target values''.
We are given $n$ input-target samples 
$S=\{(\boldsymbol{x}_i,y_i)\}_{i=1}^n$ with $(\boldsymbol{x}_i,y_i)\in{\cal X}\times {\cal Y}$ drawn i.i.d. from an underlying distribution~${\cal D}$.
We assume  for simplicity that 
for $(\boldsymbol{x},y)$ sampled from ${\cal D}$, we have
$\|\boldsymbol{x}\|=1$ and $|y|\leq m/\mu$.
We train the NN by GD
over~$S$. We assume that, initially, 
the weights of the first layer are {\em almost 0}, in the following sense:
\begin{equation}\label{eq:almost0}
\frac{m}{\mu\sqrt{n}}\max_{r\in[m]}\|\boldsymbol{w}(t_0)\| \rightarrow 0\ \mbox{ as }\ 
n\rightarrow \infty.
\end{equation}
This is the case for example when each $\boldsymbol{w}_r$ is initialized to
a value generated 
from ${\cal N}(\boldsymbol{0}, 2\boldsymbol{I}/m )$.
The objective of GD is to minimize the {\em quadratic loss} function
${\cal L}(\boldsymbol{W})=\sum_{i=1}^n\frac{1}{2}|f(\boldsymbol{W},\boldsymbol{x}_i)-y_i|^2$.
Let us define the error $v_i(t)\in\mathbb{R}$
for $t\in\mathbb{R}_{\geq 0}$ and $i\in[n]$
by 
\begin{equation}\label{eq:v}
    v_i(t)=f(\boldsymbol{W}(t),\boldsymbol{a},\boldsymbol{x}_i)-y_i.
\end{equation}
Let $\boldsymbol{v}(t)=(v_1(t),\dots,v_n(t))^\top\in\mathbb{R}^{n}$
be the {\em training error vector}.
Given an initial time $t_0$, let
$t_k=t_{k-1}+\eta_k$ where $\eta_k$ 
is the learning rate  used at step~$k\in\mathbb{N}_{>0}$.
The $k$-th step of GD 
is defined for
$r\in[m]$ by
the difference equation:
\begin{align}
\boldsymbol{w}_r(t_{k})-\boldsymbol{w}_r(t_{k-1}) & \qquad\nonumber\\
=-\eta_k & \sum_{i=1}^n|v_i(t_{k-1})|
\frac{\partial f(\boldsymbol{W},\boldsymbol{a},\boldsymbol{x}_i)-y_i}{\partial \boldsymbol{w}_r}\qquad\nonumber\\
= -\frac{\eta_k}{\mu} & \sum_{i=1}^n|v_i(t_{k-1})|a_r \zeta'(z_i)
\frac{\partial z_i}{\partial \boldsymbol{w}_r}\qquad\label{eq:7}
\end{align}
where $z_i$ denotes $\boldsymbol{x}_i^\top \boldsymbol{w}_r(t_{k-1})$,
and $\zeta'$ the derivative of $\zeta$.
Using the 1-Lipschitzness of $\zeta$ and the fact that $|a_r|=\|\boldsymbol{x}_i\|=1$,
it follows from
(\ref{eq:7}):
\begin{equation}\label{eq:w}
\|\boldsymbol{w}_r(t_{k})-\boldsymbol{w}_r(t_{k-1})\|\leq \frac{\eta_k}{\mu}\|\boldsymbol{v}(t_{k-1})\|_1,\ \ \forall r\in[m].
\end{equation}
Using Equation~(\ref{eq:4}), it follows from (\ref{eq:w}) for $k=1$:
\begin{equation}\label{eq:fW}
|f(\boldsymbol{W}(t_1), \boldsymbol{a}, \boldsymbol{x})|\leq 
\frac{m}{\mu}\max_{r\in[m]} \boldsymbol{w}_r(t_{0})+\frac{m}{\mu^2}\eta_1\|\boldsymbol{v}_0\|_1.
\end{equation}
As shown in \cite{Du} (Section 3), the discrete dynamics of
the error $\boldsymbol{v}(t)$ writes in a compact way
\begin{equation}\label{eq:8}
\boldsymbol{v}(t_{k})-\boldsymbol{v}(t_{k-1})=-\eta \boldsymbol{H}[\boldsymbol{W}(t_{k-1})]\boldsymbol{v}(t_{k-1}),\ \ \ \boldsymbol{v}(t_0)=\boldsymbol{v}_0
\end{equation}
where $\boldsymbol{H}[\boldsymbol{W}]$ 
is the NTK matrix
(see~\cite{JacotHG18})
defined 
as the $n\times n$ matrix with $(i,j)$-th entry
\begin{equation*} 
\boldsymbol{H}_{ij}[\boldsymbol{W}]=
\frac{\partial f(\boldsymbol{W},\boldsymbol{a},\boldsymbol{x}_i)}{\partial \boldsymbol{W}}^\top
\frac{\partial f(\boldsymbol{W},\boldsymbol{a},\boldsymbol{x}_j)}
{\partial \boldsymbol{W}}.
\end{equation*}
\subsection{Rademacher complexity and generalization error}
The {\em population loss} $L_{\cal D}$ over data distribution ${\cal D}$
and the {\em empirical loss} $L_S$ over 
$S=\{(\boldsymbol{x}_i,y_i)\}_{i=1}^n$ 
are defined as follows
(see \cite{AroraDHLW19}):
$$L_{{\cal D}}[f]=\mathbb{E}_{(x,y)\sim {\cal D}}[\ell(f(\boldsymbol{x}),y)],$$
$$L_{S}[f]=\frac{1}{n}\sum_{i=1}^n\ell(f(\boldsymbol{x}_i),y_i)$$
where $\ell(\cdot,\cdot)$ is the elementary quadratic function
defined by $\ell(z,y) := |z-y|^2$.
Note that we have:
\begin{equation}\label{eq:xi}
L_S[f]=\frac{1}{n}\|\boldsymbol{v}\|^2.
\end{equation}
Let $\boldsymbol{\epsilon}=(\epsilon_1,\dots,\epsilon_n)$,
${\cal T}_k=\{t_0,t_1,\dots,t_k\}$, and
$M_k\in\mathbb{R}_{>0}$ such that:
\begin{equation} \label{eq:Mk}
|f(\boldsymbol{W}(t),\boldsymbol{x})-y|\leq M_k,\ \ \ \ 
\forall (\boldsymbol{x},y)\in{\cal X}\times{\cal Y},\ t\in {\cal T}_k.
\end{equation}
The generalization loss refers to $L_{\cal D}[f]-L_S[f]$ for
the learned function $f$ given sample $S$. Given a class ${\cal F}_k$
of functions $f(\boldsymbol{W}(t),\boldsymbol{x}): {\cal T}_k\times {\cal X}\rightarrow {\cal Y}$, the notion of Rademacher complexity (denoted
${\cal R}_S({\cal F}_k)$) is useful to derive an upper bound for
the generalization loss (see \cite{ChenM2019}). 
We have:
\begin{proposition}\label{prop:1} (cf. Theorem 11.3, p. 270 of \cite{ChenM2019})
With probability at least $1-\delta$ over a sample~$S$ of size~$n$:
$$\sup_{f\in{\cal F}_k} \{L_{{\cal D}}[f]-L_S[f]\}\leq 4M_k{\cal R}_{S}({\cal F}_k)
+3M_k^2 \sqrt{\frac{log \frac{2}{\delta}}{2n}}$$
where 
$${\cal R}_S({\cal F}_k):=\frac{1}{n}\mathbb{E}_{\boldsymbol{\epsilon}\in\{\pm 1\}^n, t\in {\cal T}_k}[\sup_{f\in{\cal F}_k}\sum_{i=1}^n\epsilon_if(\boldsymbol{W}(t),\boldsymbol{x}_i)].$$
%
\end{proposition}
%
Besides, we have:
%
\begin{proposition}\label{prop:2} \cite{BartlettM02}\cite{Ma22} (Theorem 5.7). 
For a network 
with one
hidden layer,
output weights $a_r\in\{-1, 1\}$, normalized inputs $\|\boldsymbol{x}\|=1$
and $1$-Lipschitz activation $\zeta$ with $\zeta(0)=0$, a bound on the Rademacher complexity  is
$${\cal R}_S({\cal F}_k)\leq \Phi c_k$$
with
\begin{equation}\label{eq:Phi}
\Phi=\frac{2m}{\mu \sqrt{n}}
\end{equation}
\footnote{If we take $\mu=\sqrt{m}$ as in \cite{Du,AroraDHLW19}, we have $\Phi=2\sqrt{m/n}$.} and $c_k
 =\max_{r\in[m], t\in {\cal T}_k}\|\boldsymbol{w}_r(t)\|$.
\end{proposition}
\begin{proof} 
For the sake of self-containment, a proof is given in Appendix (Section~\ref{app:final}).
\end{proof}
It then follows from Propositions \ref{prop:1} and \ref{prop:2},
and~(\ref{eq:xi}):
\begin{proposition}\label{prop:Rademacher}
We have with probability at least $1-\delta$ over the sample $S$ of size $n$:
\begin{equation}\label{eq:LD0}
L_{{\cal D}}(t_k) \leq  L_{{\cal D}}^{*}(t_k)+3M_k^2\sqrt{\frac{\log \frac{2}{\delta}}{2n}}
\end{equation}
where:
$L_{{\cal D}}^{*}(t_k) =L_G(t_k)+ \frac{1}{n}\|\boldsymbol{v}(t_k)\|^2$
\ \ with
\begin{equation}\label{eq:Lgen}
L_G(t_k)=4M_kc_k\Phi.
\end{equation}
\end{proposition}
Note that $L_G(t_0)$ is almost 0 (due to Assumption~(\ref{eq:almost0})).
Note also that, in order to satisfy (\ref{eq:Mk}) for $k=1$,
we can take
\begin{equation}\label{eq:M1bis}
M_1= \frac{m}{\mu}
\left(1+\frac{\eta_1\|\boldsymbol{v}_0\|_1}{\mu}\right),
\end{equation}
using (\ref{eq:almost0}), (\ref{eq:fW}) and the fact that $|y|\leq m/\mu$
 for all $y\in{\cal Y}$.
Such an estimate of $M_1$ is very conservative. For a 
more accurate estimate of $M_1$, we can use a Monte Carlo method
(at the price of introducing a new source of probability $\varepsilon$).
\subsection{Sketch of the method}
Equation~(\ref{eq:LD0}) tells us that,
in order to get an upper bound on $L_{{\cal D}}$ (with probability $1-\delta$),
it suffices to obtain an upper bound on $L_{{\cal D}}^{*}=L_G+ L_{S}$.
We first compute an upper bound $\nu$ on $L_G$ (see Theorem~\ref{prop:NEWnu}),
then an upper bound $\omega$ on $\nu+L_{S}$ 
(so $\omega\geq L_{{\cal D}}^{*}$).
We then determine a factor $\gamma$ which, under certain condition
(see Equation~(\ref{def:gamma})), guarantees
that $\omega$ decreases by at least a quantity~$|\Delta|$ after one
GD step.
We thus obtain an upper bound on $L_{{\cal D}}^{*}$ of the form
$\Omega_1=\omega_0-|\Delta|$. 
The GD is stopped after one step
at time $t_{1}=t_0+\eta$.
The formal definitions of 
$\Delta$, $\gamma$, $\eta$, $\dots$
are given in  Section~\ref{sec:results}
(cf. Section~\ref{sec:table}).
\section{Upper Bound on the Population Loss after One  GD Step}\label{sec:results}
We denote by $\boldsymbol{u}_1(t), \dots, \boldsymbol{u}_n(t)$ the 
eigenvectors of the NTK matrix $\boldsymbol{H}[\boldsymbol{W}(t)]$, and by $\lambda_1(t), \dots, \lambda_n(t)$ their associated eigenvalues at instant $t$. For $i\in[n]$, the expression
$\lambda_i^-$ (resp $\lambda_i^+$) denotes a lower bound 
(resp. upper bound) of $\lambda_i(t)$ for $t\in[t_{0},\infty)$ 
where $t_0$ is the initial time\footnote{As the amplitude of the eigenvalues may have large variations during a brief transient time (especially with ReLU activation), it may be useful to take $t_0$ large enough to ensure
their stabilization
(``warming up'').}. 
Let $(\boldsymbol{u}_i(t_{0}))^\bot$ denote
the space orthogonal to $\boldsymbol{u}_i(t_{0})$.
The expression $\pi_i^\|(\boldsymbol{p})$ (resp. $\pi_i^\bot(\boldsymbol{p})$) denotes the projection 
of $\boldsymbol{p}\in\mathbb{R}^n$ on the eigenvector
$\boldsymbol{u}_i(t_{0})$ (resp. 
$(\boldsymbol{u}_i(t_{0}))^\bot$).
For $i\in[n]$, we define
$A_i\in\mathbb{R}_{\geq 0}$ and~$B_i\in\mathbb{R}_{\geq 0}$ as:
\begin{equation*} 
    A_i=\|\pi_{i}^\|(\boldsymbol{v}_0)\|,
\end{equation*}
\begin{equation*} 
    B_i=\|\pi_{i}^\bot(\boldsymbol{v}_0)\|.
\end{equation*}
%
We suppose that the eigenvalues are ordered as
$\lambda_1\geq \lambda_2\geq\cdots\geq \lambda_n$.
We consider the learning rate $\eta_1$
defined by
\begin{equation*} 
\eta_1=\frac{\beta}{\lambda_1^-}
\ \ \mbox{ for some }\ \beta\in(0,1).
\end{equation*}
We have then: $t_1=t_0+\eta_1=t_0+\beta/\lambda_1^-$.

Let $\theta_1$ be the angle between the eigenvector $\boldsymbol{u}_1(t_1)$
and $\boldsymbol{u}_1(t_{0})$. Let
%
%
%
%
$$\rho_1=1-\lambda_1^-\eta_1\in (0,1), \ \ 
\sigma_1=\eta_1\lambda_1^+\theta_1\in\mathbb{R}_{\geq 0},$$
and $\alpha_1\in\mathbb{R}_{\geq 0}$ such that:
\begin{align}
\alpha_{1} & \geq \frac{\sigma_{1}}{\rho_{1}}\left(\theta_{1} +
\frac{B_1}{A_1}\right).
\qquad \label{eq:alphai}
\end{align}
\begin{definition}\label{def:0}
Let:
\begin{align}
A_1' &=\rho_{1}(1+\alpha_{1})A_1,\qquad\label{eq:defmu1}\\
B_1'
& = B_1+\sigma_{1}A_1.\qquad\label{eq:defmu2}
\end{align}
\end{definition}
\begin{remark}\label{rk:small}
The angle $\theta_1$ between the eigenvector $\boldsymbol{u}_1(t_1)$ and $\boldsymbol{u}_1(t_0)$ is in practice
negligible because 
the speed of rotation of vectors $\boldsymbol{u_1},\dots,\boldsymbol{u_n}$
is slow, and $t_1<\frac{1}{\lambda_1^-}$ is small (typically $t_1<10^{-3}$).
It follows that $\sigma_1$ is itself negligible,
and condition (\ref{eq:alphai}) satisfied with $\alpha_1\approx 0$. 
Therefore: $A_1'\approx \rho_{1} A_1$
and $B'_1\approx B_1$. 
We have $A_1'<A_1$, which reflects the fact that
GD {\em contracts} the error vector $\boldsymbol{v}$ along the eigenvector $\boldsymbol{u}_1$.
In the following, symbols $\theta_1,\sigma_1, \alpha_1$ are kept in the formulas 
for the sake
of formal correctness.
\end{remark}
\begin{remark}\label{rk:rho}
Note that we have:
$\rho_{1}=1-\eta_1\lambda_{1}^-=1-\beta.$
The quantity $1-\beta$ thus corresponds to
the {\em rate of contraction} of the error
vector~$\boldsymbol{v}$ along the
eigenvector $\boldsymbol{u}_{1}$ (see Remark~\ref{rk:small}).
\end{remark}
We now define an expression $\gamma_1$ 
which, when less than $1-\beta/2$,
guarantees that 
(an overapproximation of) $L_{{\cal D}}$ decreases 
during the first step of GD
(see Remark~\ref{rk:<}).
\begin{definition}\label{def:4}
Let
\begin{equation} \label{def:gamma00}
\gamma_1=\frac{2M_1\|v_0\|_1}{(1+\alpha_1)^2\lambda_1^- A_1^2}
\frac{\Phi n}{\mu}
=\frac{4\sqrt{n}M_1\|v_0\|_1}{(1+\alpha_1)^2\lambda_1^- A_1^2}\frac{m}{\mu^2}.
\end{equation}
\end{definition}
Note that, if we take $\mu=\sqrt{m}$ as in \cite{Du,AroraDHLW19},
$\gamma_1$ is independent of $m$.
Suppose now:
\begin{equation}\label{def:gamma}
\gamma_1<1-\frac{\beta}{2}.
\end{equation}
%
%
%
Note that, from~(\ref{eq:M1bis}) and $\eta_1=\beta/\lambda_1^-$, we can take:
\begin{equation}\label{eq:fWbis}
M_1= \frac{m}{\mu}
\left(1+\frac{\beta}{\lambda_1^-}\frac{\|\boldsymbol{v}_0\|_1}{\mu}\right).
\end{equation}
\begin{proposition}\label{lemma:1}
We have:
\begin{align*}
A'_{1} &\geq  \|\pi_{1}^\|(\boldsymbol{v}(t_{1}))\|, \qquad\\ 
B'_{1} &\geq  
\|\pi_{1}^\bot(\boldsymbol{v}(t_{1}))\|. \qquad 
\end{align*}
\end{proposition}
\begin{proof} See Appendix (Section~\ref{app:1}).\end{proof}
It follows from Proposition~\ref{lemma:1}
\begin{equation}\label{eq:LS}
\frac{1}{n}\|v(t_1)\|^2\leq \frac{1}{n}\left((A'_{1}(t_{0}))^2+(B'_{1}(t_{0}))^2\right).
\end{equation}
Let us now define an overapproximation $\nu$ of $L_G$.
\begin{definition}\label{def:nu}
Let:
\begin{align}
\nu_0 & =4M_1\Phi\max_{r\in[m]}\|\boldsymbol{w}_r(t_0)\|,\qquad
\nonumber\\
\nu_1 & =\nu_0+4M_1\frac{\Phi}{\mu}\eta_1 \|\boldsymbol{v}_0\|_1.\qquad
\label{eq:nu}
\end{align}
Note that $\nu_0$ is almost 0 (due to Assumption (\ref{eq:almost0})).
\end{definition}
\begin{theorem}\label{prop:NEWnu}
(Upper bound on $L_{G}$).
We have:

$ L_G(t_1)\leq \nu_1$.
\end{theorem}
\begin{proof}
From Equation~(\ref{eq:w}), we have:
\begin{equation}\label{eq:rad00}
\max_{r\in[m]}\|\boldsymbol{w}_r(t_1)\|\leq \max_{r\in[m]}\|\boldsymbol{w}_r(t_0)\|+\eta_1\frac{\|\boldsymbol{v}_0\|_1}{\mu},
\end{equation}
hence:
\begin{align*} 
4M_1\Phi\max_{r\in[m]}\|\boldsymbol{w}_r(t_1)\| & \leq 4M_1\Phi\max_{r\in[m]}\|\boldsymbol{w}_r(t_0)\|\qquad\\
& +4M_1\frac{\Phi}{\mu}\eta_1\|\boldsymbol{v}_0\|_1,
\end{align*}
i.e., using Equation~(\ref{eq:Lgen}) and Definition~\ref{def:nu}:

$L_G(t_{1})
\leq \nu_0+\eta_1\frac{4M_1\Phi}{\mu}\|\boldsymbol{v}_0\|_1=\nu_1$.
\end{proof}
%
\begin{definition}\label{def:omega}
Let us define  
\begin{align*}
\omega_0 & =\frac{1}{n}\|\boldsymbol{v}_0\|^2+\nu_0 \qquad\\
\omega_1 &=\frac{1}{n}\left((A'_{1})^2+(B'_{1})^2\right)+\nu_1.\qquad
\end{align*}
\end{definition}
Note that, from Definition~\ref{def:omega} and Equation~(\ref{eq:LS}),
we have:
\begin{equation}\label{eq:omega1}
\nu_1+\frac{\|\boldsymbol{v}(t_1)\|^2}{n}\leq \omega_1.
\end{equation}
Using $L_{{\cal D}}^{*}=L_G+\frac{1}{n}\|\boldsymbol{v}\|^2$,
Theorem~\ref{prop:NEWnu} and (\ref{eq:LS}), 
we have:
\begin{theorem}\label{prop:NEWomega}
(Upper bound on $L_{{\cal D}}^{*}$).
We have:\ \ \ 
$L_{{\cal D}}^{*}(t_1)\leq \omega_1$.
\end{theorem}
We can now give our main result.
\begin{theorem}\label{th:gap}
For $\beta\in(0,1)$ such that~(\ref{def:gamma}) holds,
we have:\ \ \ 
$\omega_1-\omega_0\leq \Delta_1$\ \ \
with
\begin{equation}\label{eq:Psi1}
\Delta_1= - \frac{2A_{1}^2}{n}\beta\left(1-\gamma_{1}-\frac{\beta}{2}\right)(1+\alpha_{1})^2+\frac{D_1}{n}
\end{equation}
where
\begin{equation*}
D_1=A_{1}^2\left(2\alpha_{1}+\alpha_{1}^2+\sigma_{1}^2 +2\sigma_{1} \frac{B_{1}}{A_{1}}\right).
\end{equation*}
\end{theorem}
\begin{proof} See Appendix (Section~\ref{app:2}).\end{proof}
\begin{remark}\label{rk:<}
We have noted in Remark~\ref{rk:small} that, in practice,
$\alpha_1$ and $\sigma_1$  are negligible. It follows
that $\frac{D_1}{n}$ is itself negligible.
Hence $\Delta_1$ is negative because its dominant subterm
is $- \frac{A_{1}^2}{n}(1-\gamma_{1})^2(1+\alpha_{1})^2 <0$.
So $\omega_1\leq \omega_0-|\Delta_1|$.
\end{remark}
Let us now define:
\begin{equation}\label{eq:Omega}
\Omega_1:=\omega_0+\Delta_1.
\end{equation}
It follows from Theorems~\ref{prop:NEWomega} and~\ref{th:gap}:
\begin{equation}\label{eq:newfondam}
L_{\cal D}^{*}(t_1)\leq \omega_1\leq \Omega_1.
\end{equation}
Note that we have $L_G(t_1)\leq \nu_1$,
$\nu_1\leq \omega_1$ and $\omega_1\leq \Omega_1$,
according to Theorem~\ref{prop:NEWnu},
Definition~\ref{def:nu} and Equation~(\ref{eq:newfondam}) respectively,
which leads to:
\begin{equation}\label{eq:final}
L_G(t_1)\leq \nu_1\leq\omega_1\leq \Omega_1.
\end{equation}
From Proposition~\ref{prop:Rademacher} and~(\ref{eq:newfondam}), we then have:
\begin{theorem}\label{prop:final}
For $\beta\in(0,1)$ such that~(\ref{def:gamma}) holds,
we have with probability at least $1-\delta$:
\begin{equation}\label{eq:LD}
L_{{\cal D}} \leq  \Omega_1+3M_1^2\sqrt{\frac{\log \frac{2}{\delta}}{2n}}.
\end{equation}
\end{theorem}
\begin{remark}\label{rk:M1}
Let us recall that $M_1$ is an upper bound on
$sup_{(x,y)\in{\cal X}\times {\cal Y}}\max_{t\in\{t_0,t_1\}}f(W(t),x)-y$
(see (\ref{eq:Mk}).
There are two possibilities for getting a value for $M_1$. 
We can first choose a value for $\beta$, and obtain
an estimate for $M_1$
using~$(\ref{eq:fWbis})$. 
As this estimate is often too much conservative, it
is often advantageous to estimate $M_1$ using a probabilistic Monte Carlo
method. In this case, the evaluation of $M_1$ does not depend on $\beta$, 
and we choose a value $\beta=\beta^*$  
that {\em maximalises} $|\Delta_1|$ 
(i.e., {\em minimizes} $\Omega_1$). It is easy
to see that $\beta^*=1-\gamma_1$ (which satisfies (\ref{def:gamma}):
$\gamma_1<1-\beta^*/2$).
We have then:
$\eta_1= (1-\gamma_1)/\lambda_1^-$
and $|\Delta_1|\approx A^2(1-\gamma_1)^2/n$.
(See Example~\ref{ex:1}.)
\end{remark}
\begin{remark}
Note that, in the case of normalized output weight vector ($\|\boldsymbol{a}\|/\mu=1$),
the computation of 
$\Omega_{1}$ is {\em independent of~$m$}, relying on the knowledge
of the error vector $\boldsymbol{v}_0$ and the 
NTK matrix~$\boldsymbol{H}$ alone. We have:
$\Omega_1\approx \|\boldsymbol{v}_0\|^2/n-|\Delta_1|$,
and the one-step stopping time is $t_1=t_0+\beta/\lambda_{1}^-$.
\end{remark}
\begin{remark}\label{rk:only}
One can think about iterating the method, taking~$t_1$
as a new initialization time, and performing a second GD step. 
The definition of $\gamma_1$ 
(see (\ref{def:gamma00}))
then becomes:
\begin{equation} \label{eq:gamma(2)}
\gamma_2=\frac{4M_2\sqrt{n}\|\boldsymbol{v}(t_1)\|_1}{(1+\alpha_2)^2\lambda_2^- (A_2)^2}\frac{m}{\mu^2}
\end{equation}
where $A_2$ is the norm of the projection of $\boldsymbol{v}(t_1)$
on $\boldsymbol{u}_2(t_1)$.
In practice in the underparameterized context, we observe that 
$\gamma_2$s is larger than~1, which makes condition
(\ref{def:gamma}) impossible to satisfy. 
A possible interpretation is that $1-\gamma$ acts as
an ``indicator of decrease'' which is 
sensitive enough to detect the {\em sharp drop} of $L_{{\cal D}}$ 
at first step,
but not the {\em slow decrease} that follows.
The case of overparameterized networks is discussed hereafter.
\end{remark}
\subsection*{Overparameterized case}
Let us consider the case $\mu=m$.
For $m$ large enough,
all the eigenvalues of the NTK matrix become {\em positive}.
Besides, at the first GD step, 
the factor $\gamma$ tends to $0$ as $m\rightarrow\infty$
(see Equation~(\ref{def:gamma00})). 
This means that the decrease of $\Omega$ is $\Delta_1\approx A^2/n$
where $A=\|\pi_1(\boldsymbol{v}(t_0))\|$.
One then observes, that unlike the situation described in Remark~\ref{rk:only},
the quantity $\gamma_2$ involved in the second step (see Equation~(\ref{eq:gamma(2)}))
is itself close to $0$. A second step is thus possible and leads to a further
decrease of $\Omega$ by $\Delta_2\approx (A_2)^2/n$ 
with $A_2=\|\pi_2(\boldsymbol{v}(t_1))\|\approx\|\pi_2(\boldsymbol{v}(t_0))\|$.
And so on, after $n$ GD steps, $\Omega$ has decreased
of $\sum_{k=1}^{n}\Delta_k\approx \sum_{k=1}^{n} (A_k)^2/n
\approx\sum_{k=1}^n\|\pi_k(\boldsymbol{v}(t_0))\|^2/n=\|\boldsymbol{v}(t_0)\|^2/n$.
This means that, after $n$ GD steps, $\Omega_n\approx 0$,
which implies that the population loss $L_{{\cal D}}$ is itself almost null. 
This shows (unformally) that there is no overfitting
in the case of overparameterized networks:
Training decreases not only the empirical loss $L_S$ 
but also the generalization loss $L_G$ to $0$. This suggests
a new way of explaining the phenomenon of ``benign overfitting'' (see \cite{Bartlett2020}).
\begin{example}\label{ex:1}
In order to illustrate the above theoretical results, we consider 
the Van der Pol oscillator as treated in~\cite{XavierCF23}. 
The system
possesses two states $x_1$ and $x_2$, where $x_1$ is the oscillator
position, $x_2$ its velocity and $u$ a control action. The system is defined by
\begin{equation*}
\begin{cases}
\dot{x}_1=x_2\\
\dot{x}_2=(1-x_1^2)x_2-x_1+u.
\end{cases}
\end{equation*}
The goal is to design an NN that mimics the behavior of an MPC controller
that steers the system to a desired trajectory $x^{ref}$. The NN used to simulate
the MPC has $d=3$ entries $x_1$, $x_2$,  $x^{ref}$ in the input layer , $m=10$ neurons in the
hidden layer, and one output $u$ in the last layer. 
It was implemented
using PyTorch~\cite{PaszkeGMLBCKLGA19}.  
We take $t_0=0$, $\mu=m=10$ and 
each  $\boldsymbol{w}_r(0)$ ($r\in[m]$)
is independently generated
from ${\cal N}(\boldsymbol{0}, 2\boldsymbol{I}/m )$.
The NN is trained repeatedly on a 
set~$S$ of $n=800$ samples randomly selected
from the MPC data, and a $\tanh$ activation function.
We obtain the following (average) values: 
$\|\boldsymbol{v}_0\| = 17.7$, $\|\boldsymbol{v}_0\|_1=429$,
$\Omega_0= 0.392$, $\nu_0=0$ and
$\lambda_1^-=34.6$.
The norms $A_1$ and~$B_1$ of the projections of $\boldsymbol{v}_0$ 
on $\boldsymbol{u}_1(t_0)$ and $(\boldsymbol{u}_1(t_0))^\bot$ are $13.9$
and $0.35$ respectively.
The rotation angle $\theta_1$ is negligible, so are~$\alpha_{1}$ and~$D_1$.
A Monte Carlo method
gives $M_1\leq 0.25$
with probability at least $1-\varepsilon$\footnote{More precisely,
we find that, for $n'=500$ trials and a standard deviation~$s$,
$M_1$ is in the interval $[0.21,0.25]$
with  probability $\varepsilon=s/10$.}.
We then use $\beta=1-\gamma_1$ (see Remark~\ref{rk:M1}).
The values of $\gamma_{1}$, $\beta$, $\|\boldsymbol{v}(t_1)\|^2/n$, $\Omega_1$ and $t_1=\eta_1=(1-\gamma_1)/\lambda_1^-$ are given in Table~I. 
From Theorem~\ref{prop:final}, it follows that,
with probability $1-\delta-\varepsilon$, we have: 
$L_{{\cal D}} \leq  0.31  +0.19\sqrt{\log \frac{2}{\delta}}$.
%
\begin{table}
\begin{center}
\begin{tabular}{|c|c|c|c|c|c|}
\hline  
$M_1$ & $\beta$ & $\gamma_{1}$ & $\|\boldsymbol{v}(t_1)\|^2/n$ & $\Omega_{1}$ & $t_{1}$ \\  
\hline 
$0.25$ & $0.82$ & $0.18$  & $0.19$ & $0.31$ & $0.024$\\  
\hline
\end{tabular}
\end{center}
\label{table:1}
\caption{
Values of $M_1$, $\beta$, 
$\gamma_{1}$, $\|\boldsymbol{v}(t_1)\|^2/n$, $\Omega_{1}$, $t_{1}$}
\end{table}

The curves of $\Omega$ (red) are drawn 
$\omega$ (black),
$\frac{1}{n}\|\boldsymbol{v}\|^2$ (green),
$\nu$ (blue) and $L_G$ (purple) are drawn for $t\leq t_{1}=0.024$
in Figure~\ref{fig:Th5}.
We see $L_G(t_1)\leq \nu_1\leq\omega_1\leq \Omega_1$
in accordance with formula (\ref{eq:final}).
Note that 
$\omega_1\approx \nu_1$ because
$\omega_1-\nu_1\approx\left((\gamma_1 A_{1})^2+B_{1}^2\right)/n=
\left((0.18\times 13.9)^2+0.35^2\right)/800
=0.008\ll \nu_1=0.11$.
%
%

In order to evaluate the tightness of upper bound $\Omega_1$
and the relevance of stopping time $t_1$, we also compute 
an empirical estimate of 
$L_{{\cal D}}$
via the following ``test loss'' 
$L_{test}(t)=\sum_{i=1}^{n_{test}}(f(\boldsymbol{W}(t),a,\boldsymbol{x}_i^{test})-y_i^{test})^2/n_{test}$\\
where
$\{(\boldsymbol{x}_i^{test},y_i^{test})\}_{i=1}^{n_{test}}$ 
is a separate set of $n_{test}=5000$ samples.
The test loss $L_{test}(t)$ reaches a minimum equal to $0.054$
at 
$t=t^*=0.08$. The upper bound
$\Omega_{1}$ is thus around  {\em 6 times larger} than~$L_{test}(t^*)$,
and $t_{1}$ 
{\em 3 times shorter}  than~$t^*$
(see Figure~\ref{fig:Th6}).
Note that, at 
$t=t_1=0.024$, $L_{test}$ has fallen sharply from  $0.39$ to 
$0.10$,
a value relatively close to the minimum $0.054$. This indicates
that 
$L_{{\cal D}}$ mainly decreases
during the {\em first} GD step,  and confirms the relevance
of the ``one-step''  stopping strategy.
\begin{figure}[h!]
\centering
\includegraphics[scale=0.32]{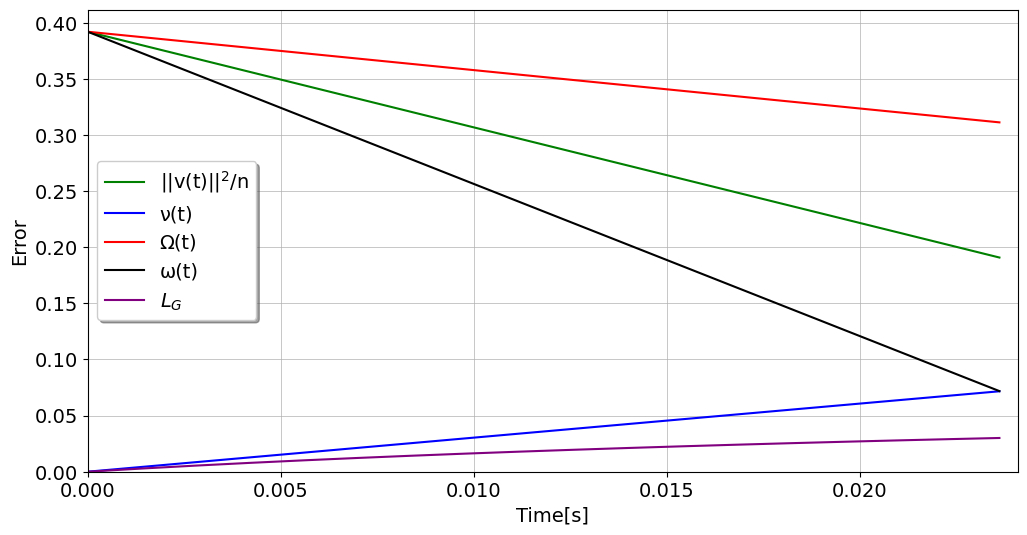}
\caption{Curves $\Omega$, $\frac{\|\boldsymbol{v}\|^2}{n}$,
$\omega$, $\nu$, $L_G$
for $t\leq t_1=0.024$} 
\label{fig:Th5}
\end{figure}
\begin{figure}[h!]
\centering
\includegraphics[scale=0.32]{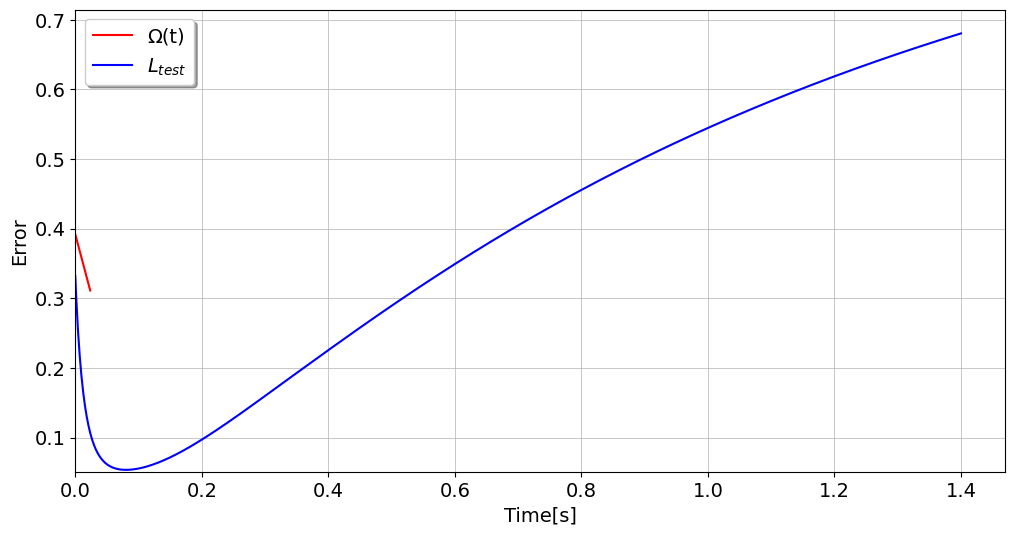}
\caption{Curves $\Omega$ (for $t\leq t_1=0.024$) and $L_{test}$} 
\label{fig:Th6}
\end{figure}

\end{example}
\section{Final Remarks}
%
Using the theories of NTK matrix and Rademacher complexity,
we obtained analytical formulas 
for a {\em one-step} stopping time
and an  upper bound $\Omega_1$ on the population loss. 
The computation of 
$\Omega_{1}$ is {\em independent of $m$}
(at least in the case of normalized output weight vector),
relying on the knowledge
of the {\em initial} error vector $\boldsymbol{v}_0$ and 
NTK matrix~$\boldsymbol{H}$ alone. 
On the example of a Van der Pol oscillator,
the simulations are consistent with the theoretical results.
Our method also suggests a new way of explaining the phenomenon
of ``benign overfitting'' in the overparameterized context.
Our method is however limited to NNs 
with a single hidden layer, a fixed output layer and 
a scalar output. 
In future work, we plan to extend the method to networks with more than one hidden layer
and multi-dimensional outputs.
%
\newcommand{\LNCS}{LNCS}

\ifdefined\VersionAuthor
	\renewcommand*{\bibfont}{\small}
	\printbibliography[title={References}]
\else
	\bibliographystyle{IEEEtran} 
	\bibliography{ecc25}
\fi
\section{Appendix: Summary of Definitions}\label{sec:table}
The fact that $\Omega_{1}$
is an upper bound on $L_{{\cal D}}^*$ relies on the sequence of inequalities:
$$\Omega_1
\geq \omega_{1}\geq L_S(t_{1})+\nu_{1}\geq L_S(t_{1})+L_G(t_{1})=L_{{\cal D}}^*$$
where the
definitions of $\omega_1$, $\nu_1$, $L_s$ and $L_G$
summarized in Table~II. 
Auxiliary definitions
are:
\begin{itemize}
%
%
\item 
$t_1=t_0+\eta_1=t_0+\beta/\lambda_{1}^-$,
\item $\Phi=2m/(\mu \sqrt{n})$,
\item $A'_{1}\approx (1-\beta)\|\pi_1^\|(\boldsymbol{v}_0)\|$,
\ $B'_{1}\approx \|\pi_1^\bot(\boldsymbol{v}_0)\|$
\item $|\Delta_1|\approx 2A_1^2\beta(1-\gamma_1-\beta/2)/n$.
\end{itemize}
where $\beta\in(0,1)$ and $\gamma_1$ (given by Definition~\ref{def:4})
is such that $\gamma_1<1-\beta/2$.
%
\begin{table}
\begin{tabular}{|c|c|c|c|}
\hline  
$L_S(t)$ & $L_G(t)$ & $\omega_1-\nu_1$ & $\nu_1-\nu_{0}$ \\  
\hline 
&\  & \ & \ \\
$\|\boldsymbol{v}(t)\|^2/n$ & $4M_1\Phi$ & $(A'_{1})^2/n$ & $4M_1\eta_1/\sqrt{n}$ \\ 
\ & $\times \max_r\|\boldsymbol{w}_r(t)\|$ & $+(B'_{1})^2/n$ & $\times \|\boldsymbol{v}_{0}\|_1$ \\ 
\hline
\end{tabular}
\label{table:3}
\caption{Definition of $L_S(t)$, $L_G(t)$, $\omega_1-\nu_1$, $\nu_1-\nu_{0}$}
\end{table}
\section{Appendix: Proof of Proposition \ref{lemma:1}}\label{app:1}
For the simplicity of notation, we omit the symbol index~$1$, so $\alpha_{1}, \rho_{1}, A_{1}, B_{1}, \theta_1, \dots$ write $\alpha, \rho, A, B, \theta, \dots$.
For $k=1$, Equation~(\ref{eq:8}) writes:
\begin{equation}\label{eq:H(eta)}
    \boldsymbol{v}(t_1)=(\boldsymbol{I}_n-\eta_1 \boldsymbol{H}[\boldsymbol{W}(t_0)])\boldsymbol{v}_0
\end{equation}
where $\boldsymbol{H}[\boldsymbol{W}(t)]$ is
the NTK matrix at time $t$. 
The matrix $\boldsymbol{H}[\boldsymbol{W}(t_0)]$ can be written:
$\boldsymbol{H}[\boldsymbol{W}(t_0)]=\boldsymbol{P} \boldsymbol{\Lambda} \boldsymbol{P}^\top$
where $\boldsymbol{\Lambda}$ is the diagonal $n\times n$ matrix having $\lambda_1(t_{0}),
\dots, \lambda_n(t_{0})$ as diagonal elements, and $\boldsymbol{P}$ the transition matrix
expressing $\{\boldsymbol{v}(t_1),(\boldsymbol{v}(t_1))^\bot\}$ 
in the basis $\{\boldsymbol{v}(t_0),(\boldsymbol{v}(t_0))^\bot\}$.

The  abstraction $\boldsymbol{H}_{abs}$ of $\boldsymbol{H}[\boldsymbol{W}(t_{0})]$ is the matrix
defined as 
$\boldsymbol{H}_{abs}=\boldsymbol{P} \boldsymbol{\Lambda}_{abs} \boldsymbol{P}^\top$
where $\boldsymbol{\Lambda}_{abs}$ is the $n\times n$ matrix having all its entries $0$
except the $(1,1)$-entry equal to $\lambda_{1}(t)$.
Since all the entries  of  $\boldsymbol{I}_n-\eta_1\boldsymbol{\Lambda}$
are non-negative and not larger than the corresponding
entries of~$\boldsymbol{I}_n-\eta_1\boldsymbol{\Lambda}_{abs}$,
it follows from
Equation~(\ref{eq:H(eta)}):
\begin{align}
\boldsymbol{v}(t_1) & = \boldsymbol{P}(\boldsymbol{I}_n-\eta_1 \boldsymbol{\Lambda}) \boldsymbol{P}^\top \boldsymbol{v}(t_0)\qquad\nonumber\\
&\boldsymbol{\leq} \boldsymbol{P}(\boldsymbol{I}_n-\eta_1\boldsymbol{\Lambda}_{abs})\boldsymbol{P}^\top \boldsymbol{v}(t_0)\qquad\label{eq:first0}
\end{align}
where $\boldsymbol{p}=(p_1,\dots,p_n) \boldsymbol{\leq} \boldsymbol{q}=(q_1,\dots,q_n)\in\mathbb{R}^n$ 
means $p_i\leq q_i$ for all $i\in[n]$. Equation~(\ref{eq:first0}) can be written in a compact way:
\begin{align}
(v_1(t_1),v_{\neq 1}(t_1)) & 
\boldsymbol{\leq} \boldsymbol{P}_2(\boldsymbol{I}_2-\eta_1\boldsymbol{\Lambda}^*_{abs})\boldsymbol{P}_2^\top (v_1(t_0),v_{\neq 1}(t_0))\qquad\label{eq:first}
\end{align}
where $v_{1}(t)=\|\pi_1^\|(\boldsymbol{v}(t))\|$, $v_{\neq 1}(t)=\|\pi_1^\bot(\boldsymbol{v}(t))\|$, 
$\boldsymbol{\Lambda}_{abs}^*$ is the $2\times 2$-diagonal matrix
having $\lambda_{1}^-$ and $0$ as 
first and second diagonal elements respectively,
and $\boldsymbol{P}_2$ is a rotation $2\times 2$-matrix of 
angle $\theta$ between $\boldsymbol{u}_1(t_0)$
and $\boldsymbol{u}_1(t_1)$ (i.e., the angle
between $\pi_1(\boldsymbol{v}(t_0))$ and $\pi_1(\boldsymbol{v}(t_{1}))$):
\[\boldsymbol{P}^{*}=
\left [{\begin{array}{cc}
\cos\theta & -\sin\theta\\
\sin\theta & \cos\theta
\end{array}}\right].
\]
It follows from~(\ref{eq:first}):
\begin{align}
v_1(t_1) & \leq (1-\lambda^-\eta_1 \cos^2\theta)v_1(t_0) \qquad\nonumber\\
&\ \ \ \ -\lambda^-\eta_1\sin\theta\cos\theta v_{\neq 1}(t_0),\qquad \label{eq:i1bis}
\end{align}
\begin{align}
v_{\neq 1}(t_1) & \leq (1-\lambda^-\eta_1\sin^2\theta) v_{\neq 1}(t_0) \qquad\nonumber\\
&\ \ \ \ -\lambda^-\eta_1\sin\theta\cos\theta v_1(t_0).\qquad\label{eq:i2}
\end{align}
Hence, using the facts

$\cos^2\theta=1-\sin^2 \theta$,
$\rho=1-\lambda^-\eta_1$,

$\lambda^-\leq \lambda(t_1)\leq \lambda^+$, $|\sin \theta|\leq |\theta|$,

$v_1(t_0)=\|\pi_{1}^\|(\boldsymbol{v}(t_0))\|= A(t_0)$,

$v_{\neq 1}(t_0)= \|\pi_{1}^\bot(\boldsymbol{v}(t_0))\|= B(t_0)$,\\ 
we derive from (\ref{eq:i1bis}):
\begin{align*}
\|\pi_{1}^\|(\boldsymbol{v}(t_1))\| & \leq \rho\|\pi_{1}^\|(\boldsymbol{v}(t_0))\| \qquad\\
& \ \ \ +\lambda^+\eta_1\theta^2\|\pi_{1}^\|(\boldsymbol{v}_0)\|+\lambda^+\eta_1\theta\|\pi_{1}^\bot(\boldsymbol{v}(t_0))\|\qquad \\
&\leq  \rho A(t_0)+\sigma\theta A(t_0) 
+\sigma B(t_0)\qquad \\
&\leq  \rho(1+\alpha)A(t_0)=  A'(t_0)\qquad
\end{align*}
where the last inequality comes from condition (\ref{eq:alphai}).\\
Likewise, it follows from (\ref{eq:i2}):
\begin{align*}
\|\pi_{1}^\bot(\boldsymbol{v}(t_1))\| &\leq \|\pi_{1}^\bot(\boldsymbol{v}(t_0))\|+\lambda^+\eta_1\theta \|\pi_{1}^\|(\boldsymbol{v}(t_0))\|\qquad\\
&\leq B(t_0) +\sigma A(t_0) = B'(t_0).\qquad
\end{align*}
\section{Appendix: Proof of Theorem \ref{th:gap}}\label{app:2}
As before, we omit the index~$i$ from the different symbols.
Using Equations~(\ref{eq:defmu1}), (\ref{eq:defmu2}), (\ref{eq:LS}), (\ref{eq:nu}) we have
\begin{align*}
\omega_1-\omega_0 &\leq  \nu_1-\nu_0 +\frac{1}{n}\left((A')^2+(B')^2)\right)& \qquad\\
&\ \ \ \ \ -\frac{1}{n}(A^2+B^2)&\qquad\\
&\leq 4M_1\eta_1\frac{\Phi}{\mu}\|\boldsymbol{v}_0\|_1 
-\frac{A^2}{n}(1-\rho^2(1+\alpha)^2) &\qquad\\
&\ \ \ \ +\frac{A^2}{n}\left(\sigma^2 +2\sigma \frac{B}{A}\right). &\qquad
\end{align*}
Hence, using $\rho=1-\eta_1\lambda^-$:
\begin{align}
\omega_1-\omega_0& \leq 4M_1\frac{\Phi}{\mu}\eta_1\|\boldsymbol{v}_0\|_1
-\frac{A^2}{n}2\eta_1\lambda^-(1+\alpha)^2 \qquad \nonumber\\
&\ \ \ \ \ +\frac{A^2}{n}\eta_1^2(\lambda^-)^2(1+\alpha)^2
+\frac{D_1}{n} \label{eq:rhs4bis}
\end{align}
with
$$D_1=A^2\left(\sigma^2 +2\sigma \frac{B}{A}+2\alpha+\alpha^2\right).$$
Equation~(\ref{eq:rhs4bis}) becomes using Equation~(\ref{def:gamma00}):\\
\\
$$\omega_1-\omega_0-\frac{D_1}{n}\leq  
-\frac{2A^2(1+\alpha)^2}{n}\lambda^-\eta_1\left(-\gamma+1-\frac{\lambda^-\eta}{2}\right)$$
with
$\eta_1=\beta/\lambda^-$.
Therefore:
$\omega_1-\omega_0\leq$ $\Delta_1$ with
$$\Delta_1 =- \frac{2A^2(1+\alpha)^2}{n}
\beta\left(1-\gamma-\frac{\beta}{2}\right)+\frac{D_1}{n},$$ 
i.e.~(\ref{eq:Psi1}).
\section{Appendix: Proof of Proposition~\ref{prop:2}}\label{app:final}
$$R_S=E_{\boldsymbol{\epsilon}}\left[\sup_f \frac{1}{n\mu}\sum_{i=1}^n \epsilon_i\sum_{r=1}^m a_r\zeta(\boldsymbol{w}_r^\top \boldsymbol{x}_i)\right],$$
i.e., since all the $a_r$s are equal to $\pm 1$:
\begin{equation*} 
R_S\leq \frac{1}{n\mu}E_{\boldsymbol{\epsilon}}\left[\sup_f \sum_{i=1}^n \epsilon_i\sum_{r=1}^m  |\zeta(\boldsymbol{w}_r^\top \boldsymbol{x}_i)|\right].
\end{equation*}
Thanks to Talagrand's lemma, 1-Lipschitzness of $\zeta$, it follows:
$$R_S\leq\frac{2}{\mu n} E_{\boldsymbol{\epsilon}}\left[\sup_f \sum_{i=1}^n \epsilon_i\sum_{r=1}^m  \boldsymbol{w}_r^\top \boldsymbol{x}_i\right].$$
Then, using 
$\max_{\ell\in\{0,1\}, r\in[m]}\|\boldsymbol{w}_r(t_\ell)\|\leq c_1$:
%
$$R_S\leq \frac{2m}{\mu}\left(\frac{1}{n}E_{\boldsymbol{\epsilon}}\left[\sup_{\|\boldsymbol{w}\|\leq c_1, \|\boldsymbol{x}_i\|=1} \sum_{i=1}^n \epsilon_i \boldsymbol{w}^\top \boldsymbol{x}_i\right]\right).$$
Then, using $R_{{\cal H}}\leq c_1/\sqrt{n}$
for the class 
$${\cal H}=\{\boldsymbol{w}^\top\boldsymbol{x}:
\|\boldsymbol{w}\|\leq c_1, \|\boldsymbol{x}\|=1\}$$
(see \cite{AFM2020} Theorem 11.5),
we have:
$$R_S\leq \frac{2m}{\mu}\left(\frac{c_1}{\sqrt{n}}\right)=c_1\Phi,$$
with 
$$\Phi=\frac{2m}{\mu\sqrt{n}}.$$
\end{document}